\documentclass[twoside,11pt]{article}

\usepackage[preprint]{jmlr2e}
\usepackage{blindtext}
\usepackage{amsmath}
\usepackage{lastpage}
\jmlrheading{??}{2023}{1-\pageref{LastPage}}{6/23; Revised ??/??}{??/??}{??-????}{Justin Tittelfitz}
\ShortHeadings{FPR Estimation with Label Noise}{J. Tittelfitz}
\firstpageno{1}

\hypersetup{ hidelinks }
\usepackage{courier}
\usepackage{booktabs} 

\newcommand{\set}[1]{\{#1\}}

\newcommand{\wt}[1]{\widetilde{#1}}

\begin{document}

\title{FPR Estimation for Fraud Detection in the Presence of Class-Conditional Label Noise}

\author{\name Justin Tittelfitz PhD \email tttlf@amazon.com \\
       \addr AWS Fraud Prevention\\
       Seattle, WA 98101, USA
       }

\editor{JMLR editor}

\maketitle

\begin{abstract}
We consider the problem of estimating the false-/ true-positive-rate (FPR/TPR) for a binary classification model when there are incorrect labels (label noise) in the validation set. Our motivating application is fraud prevention where accurate estimates of FPR are critical to preserving the experience for good customers, and where label noise is highly asymmetric. Existing methods seek to minimize the total error in the cleaning process - to avoid cleaning examples that are not noise, and to ensure cleaning of examples that are. This is an important measure of accuracy but insufficient to guarantee good estimates of the true FPR or TPR for a model, and we show that using the model to directly clean its own validation data leads to underestimates even if total error is low. This indicates a need for researchers to pursue methods that not only reduce total error but also seek to de-correlate cleaning error with model scores. 
\end{abstract}
\begin{keywords} Label noise, fraud and abuse detection, adversarial behavior, model validation, false positive rate, fairness, ensemble methods \end{keywords}
\section{Introduction}

We consider the problem of false positive rate (FPR) estimation as part of model validation as it relates to binary classification for fraud detection in the presence of label noise. In particular, we work in a setting where class imbalance is expected, where label noise is also imbalanced and feature-dependent, and where operational constraints are placed on metrics like FPR or precision in order to control negative impact to legitimate users. 

Formally we consider data of the form $(x,y,y^*)$ where for each example with features $x$ we have a true label $y$ and an observed, possibly different label $y^*$. If we train some model $f$ on $(x,y^*)$ and then measure the apparent FPR at some score threshold $t$ on noisy data, $p(f > t | y^* = 0)$, this may be different than the actual FPR at this same threshold, i.e. $p(f > t | y = 0)$. In later experiments (see Table \ref{tab:exact_cloud}), we will see that FPR can be overestimated by as much as $190\%$ on noisy validation data. 

Cases where $(y, y^*)=(1,0)$ represent \textit{undetected fraud} and are more prevalent than vice-versa, and among other effects, can lead to overestimates of the FPR (since there are more "observed false positives" than true false positives) . In this work, we will assume that all noise is of this form and explore the question of cleaning the validation data in order to obtain more accurate estimates of the FPR or TPR at a given score threshold during validation. We will assume that the true fraud level $p(y=1)$ is known - this is reasonable in practice if one can either sample subsets of data for label correction by experts, or by establishing historical baselines. 

We define a \textbf{cleaning method} as a function $c : (x,y^*) \mapsto \set{0,1}$ which attempts to correct observed data based on the features and observed label. A trivial example is $c(x,y^*) = y^*$, i.e. no changes are made to the observed label. Another example that can work well in some situations is to use the model $f$ that we want to validate, and then find some threshold $\tau$ above which examples may be relabelled:
\begin{align}
\label{direct-method}
    c(x,y^*) = \begin{cases}
        1, &\textrm{ if } y^* = 1 \textrm{ or } f(x) > \tau \\
        0, &\textrm{ otherwise.}
    \end{cases}
\end{align}
We will call such a method a \textbf{direct method}. In most cases, if $p(y=1)$ is known, we can consider methods $c$ that achieve $p(c=1) = p(y=1)$ (e.g. for direct methods, choosing $\tau$ appropriately). We will call such a method a \textbf{calibrated method}.

We will be interested in the total error of such a method, letting $e_1 = (y=0, c=1)$ be the type-1 error of the method, in other words the cases where the method flipped a label it shouldn't have, and $e_2 = (y=1, c=0)$ be the type-2 error of the method, the cases where the label should have been flipped but the method failed to do so; in the general setting, these errors can occur in other ways, but in our class-conditional setting, this characterizes $e_1, e_2$. It is obvious that a good cleaning method should seek to minimize both types of errors here, whether we are concerned about cleaning data for training, validation, or other purposes.

Our main result (Theorem \ref{main-thm}) is to show that, among all cleaning methods with $p(y^*=1, c=0)=0$ (i.e. those consistent with our class-conditional assumption), with $p(c=1) = p(y=1)$, and with comparable error $e_1, e_2$, that the calibrated direct method defined by \eqref{direct-method} will always lead to underestimates of FPR, in fact the \textit{worst possible} underestimate of error, and likewise the worst possible \textit{over}estimates of TPR. This happens because the cleaning error and the model score are highly correlated - cleaning errors are extremely harmful to FPR estimates in this case - and so if we as researchers are interested in mitigating the effects of label noise in FPR/TPR validation that we must find methods that not only seek to reduce total error, but also de-correlate cleaning error with the scores of the model being validated.

\section{Background}

In much of the literature, authors have focused on the effects of label noise in the training phase for multi-class classification problems (e.g. image classification) and focus on improving accuracy. While the effect on training is certainly noteworthy, these studies often omit any discussion of the effect of noise on model validation, i.e. choosing a score threshold to target a fixed FPR (or precision etc.) or estimate TPR in production. This becomes a critical concern in applications like fraud/abuse/network intrusion detection where class imbalance is usually present, and where high false positive rates can mean unacceptable legitimate customer impact, or where low TPR may make a model not worth deploying. Indeed, notions of FPR and TPR do not even necessarily make sense in a multi-class setting and so it is natural if authors focus on reduction of cleaning error or other considerations.

Direct methods like \eqref{direct-method} are then a reasonable strategy and can yield good results in many settings. In particular if $f$ is a good model, we can hope that the cleaning error is low, and this may perform well when FPR validation is not a concern. CleanLab (\cite{northcutt2021confident}) is such a method, and is a highly effective and easy-to-use way to identify noise in a dataset. This method takes $\tau = \mathbb{E}\left[f(x) : y^*=1 \right]$, i.e. the average score of examples with observed label $1$. Important to note is that this choice of $\tau$ does not ensure $p(c=1) = p(y=1)$ (in fact it will probably \textit{not} be equal), so our main result does not necessarily apply to CleanLab but it still has the drawback of cleaning error being correlated with the chance of misclassification in validation. 

As an alternative, we consider micro-models, a collaborative approach where an ensemble of weaker models votes on noise removal. Micro-models will allow us to better decouple the cleaning process from validation and we show that this leads to accurate FPR estimates in most scenarios. This method is also extremely straightforward to implement and use and is relatively parameter-free, we only need to choose the number of models in the ensemble. For our experiments we will work with our production data as well as several publicly available fraud-related datasets. These datasets are considered noise-free so we will add time-dependent noise to test our method.

Micro-models and other ensemble cleaning methods have existed in the literature for some time now. The main novelty of this work is to show the theoretical downsides of direct cleaning methods in order to suggest that researchers in this area devote effort to finding alternatives, and to provide guidance on the characteristics that such methods will exhibit (de-correlation of model scores with cleaning errors); these experiments suggest that simple methods like micro-models already show adequate results and so hopefully more sophisticated methods can make even further advances.

\subsection{Model Validation for Fraud Prevention}

While the overt goal of applying machine learning to fraud prevention (or related security/abuse applications) is to identify true positives, if these approaches are too heavy-handed or aggressive they will cause unacceptable friction or even more severe types of impact to legitimate actors that are falsely identified as suspicious. It is therefore common to constrain models to operate below a certain false positive rate (or above a certain precision), and this is commonly done by setting a score threshold during model validation. In particular, given an acceptable FPR, a typical process is to first train a model that gives real-valued predictions (scores between $0$ and $1$), then evaluate some held-out labelled data, and use these scores to generate an ROC curve. From this ROC curve we can then identify a model score that corresponds to the FPR target. From here we use this score as a decision threshold for models in production, and hopefully the actual FPR observed in the live setting is close to the FPR observed in validation. 

However if the validation process is flawed, we may under- or overestimate the actual FPR: if we underestimate the FPR then we will cause more customer impact than intended, and if we overestimate the FPR we may choose a threshold higher than necessary and allow more fraud than we would otherwise. Of course there are many reasons one may incorrectly estimate the FPR from the validation data, and some variation is to be expected, but here we will study the effect of label noise on this process. In particular if the validation set contains many true positives that are incorrectly labeled as legitimate then we will have a tendency to overestimate the FPR in validation and our TPR in production will suffer.

Of course, label noise can occur in the other direction as well, legitimate examples incorrectly labeled as fraudulent, and this would have the opposite effect on the estimation of the FPR. In practice this may be much less common as customers/account-holders/applicants are given means of escalating or disputing decisions that are made incorrectly, allowing the organization or business to reinstate accounts and correct labels.

As a final note, we mention that in the literature there is discussion of ``cleaning'' vs. ``filtering'' once noise is identified. Specifically, once one has identified that an example is noisy, they can either flip the label (cleaning) or remove the example from the dataset (filtering). In the remainder of this work, we will assume we are cleaning data rather than filtering. 

\subsection{Sources of Noise}

Actually-fraudulent cases may be incorrectly labeled for a variety of reasons. Detection is often based on a combination of attributes (what the account ``looks'' like) and behaviors (what the account does) and while we may collect information about an account's attributes as soon as it is created, behavior signals may lag behind by days or even weeks. Thus we expect ``label maturity'' to be a notable concern - if we are training a model only on attributes (e.g. to stop fraudsters immediately after account creation) we will need to account for this as we choose training and validation sets. One possible solution is to only use data that has aged suitably, but this has the drawback of discarding potentially useful data, especially in settings where fraudulent behavior is constantly evolving with time and recent data can be very valuable in understanding current behavior.

Other possible sources of noise may be the result of the adversarial nature of the problem; fraudsters are interested in evading detection by changing their attributes over time and finding combinations that are successful at fooling current models. This search for optimal attributes may leave a trail of abandoned accounts that ``look'' fraudulent but do not otherwise have behavioral signals needed to justify an enforcement action, and thus the label may be incorrect. Fraudsters may even do this intentionally in some settings, creating some accounts and deliberately abandoning them in order to ``poison'' datasets. In some of our experiments we will generate noise according to timestamps which most closely resembles the label maturity scenario.

\subsection{Existing work}

The problem of label noise is well-discussed in the literature, for overviews see the surveys \cite{johnson2022survey} and \cite{frenay2013classification}. In many cases authors are interested in multi-class classification problems such as image classification, e.g. \cite{song2020robust}, \cite{collier2021correlated}. Typical examples of label noise in this case might include labelling an image of a dog as a wolf or as a cat. In these cases, one is primarily interested in improving model accuracy, notions of a false positive rate don't make immediate sense. Further, use of neural networks has become ascendant in these applications, and these types of models seem to be more sensitive to noise in the training process, thus authors are mostly interested in cleaning data before the training phase in order to improve results. In some cases, authors focus on making the training process robust to noise, e.g. through early stopping (\cite{xia2020robust}), data augmentation (\cite{zhang2020does}), by modifying weights (\cite{majidi2021exponentiated}) or other modifications to the loss function (\cite{lyu2019curriculum}), etc.

In contrast, in our application we will be working with tabular data consisting of numerical and categorical features, and find that models like gradient boosted decision tree ensembles perform well on these tasks. As noted earlier, we will be less interested in the training phase (where we have found our models are generally robust to noise) and need methods that can directly identify and filter/clean noisy examples. While there is no existing literature (to our knowledge) on noise removal specifically in validation, there is extensive literature on generic filtering/cleaning of examples.

In many of these studies, assumptions are usually made about the nature of the noise and it is often taken to be completely independent of the examples, or at least independent of the features (e.g. noise may not be symmetric with respect to the labels but otherwise two examples with the same label have the same chance of being noise). In our case however we think this is not a realistic assumption, e.g. label maturity issues will be more prevalent among more recent examples. This has been explored in \cite{zhang2021learning}, \cite{liu2021understanding} and \cite{chen2021beyond} among others. We will then be interested in conducting experiments that have time-dependent noise.

\subsection{CleanLab}

We will compare micro-models to two direct methods, the first is the direct method in \eqref{direct-method} and the second is the related, but usually more conservative method known as CleanLab. CleanLab was introduced in 2019 (see \cite{northcutt2021confident}) and represents an evolution of techniques first discussed in \cite{northcutt2017learning}, and has also been used to explore reliability of benchmark datasets \cite{northcutt2021pervasive}, but the focus is again on accuracy. Besides these studies, CleanLab is also available as an open-source software package which makes it ideal for benchmarking against.

Briefly, CleanLab works by computing the average prediction for an example with observed labelled class $y^*=i$, then scoring all examples with given observed label $y^*=j\neq i$ and determining that any scoring above the average computed in the first step are probable noise (in our case we only consider $i = 1$ and $j = 0$, i.e. we look for ``legitimate'' examples that have riskier scores than an average example of confirmed fraud).

\section{Theoretical Results and Intuition}

In this section we prove our main result, that in the class-conditional setting, calibrated direct methods consistently underestimate FPR and overestimate TPR during validation. We show that this is due to correlation between the model score and cleaning error for direct methods, which motivates the need for alternatives, and gives intuition for why micro-models can be a simple but effective alternative in this setting. After some preliminary setup, we state a series of results leading to our main theorem, and give proofs of results in Appendix \ref{proofs}.

\subsection{Preliminaries}
Most of this section restates earlier discussion but we collect it here for easy reference. Recall that $y^*, y \in \set{0,1}$ are the observed and true (unknown) labels respectively, and let $c : (x,y^*) \mapsto \set{0,1}$ be some cleaning method; e.g. $c=1$ could mean that we relabelled an example with observed label $y^*=0$, or that $y^*=1$ to begin with and we left it alone. Note that in the class-conditional setting $p(y^*=1, y=0) = 0$ we should always take $c=1$ if $y^*=1$, i.e. $p(c=0, y^*=1) = 0$.

Suppose we also have a model $f$ which scores examples $x$, i.e. $f : x \mapsto [0,1]$. At a given model score threshold $t$, and for a given $x$ we are interested in determining whether it is an actual false positive (i.e. according to the true label $y$) as well as an estimated false positive (i.e. according to the cleaned label $c$). That is, we want to compute
\begin{align*}
	FPR_{actual} - FPR_{est} = p(f > t | y = 0) - p(f>t | c=0).
\end{align*}
Let $e_1 = (y=0, c=1)$ e.g. a type-1 error, and $e_2 = (y=1, c=0)$ denote a type-2 error\footnote{By assumption we will never clean an observed label from $y^*=1$ to $c=0$ but this type of error can still arise if we \textit{fail} to clean a case where $(y,y^*)=(1,0)$}. If $f > t$ then errors of the first type will contribute to underestimates of the FPR and vice-versa for $e_2$. 

We are assuming that the true fraud level $p(y=1)$ is known, even if we cannot identify which examples are mislabeled with perfect accuracy. Generally speaking, this may be reasonably determined by sampling a subset of data for expert label correction or via historical baselines. We say a cleaning method with $p(c=1) = p(y=1)$ is \textbf{calibrated}. Note that under the assumption $p(y=0, y^*=1) = 0$, knowing $p(y=1|y^*=0)$ is sufficient to determine $p(y=1)$:
\begin{align*}
p(y=1) = p(y=1|y^*=1)p(y^*=1) + p(y=1|y^*=0)p(y^*=0),    
\end{align*}
with $p(y=1|y^*=1)=1$ and $p(y=1|y^*=0)$ known by assumption, and $p(y^*=0), p(y^*=1)$ directly observable. That is, under a class-conditional assumption, knowing the noise level is enough to know the true fraud level.

\subsection{Main Results}\label{theory}
First, we prove that for a calibrated method, the probability of making either type of error is equal. (We provide proof of this Lemma and all other results in Appendix \ref{proofs}).
\begin{lemma}
\label{main-lemma} 
	We have $p(c=1) = p(y=1)$ if and only if $p(e_1) = p(e_2)$.
\end{lemma}

Now we can show that under this assumption, the absolute error in the FPR estimate is related to the correlation between error and the threshold selection process.
\begin{proposition}
\label{main-prop}
If $c$ is a calibrated method ($p(c=1) = p(y=1)$) then
	\begin{align*}
		FPR_{actual} - FPR_{est} &= p(c=1)\left[p(f>t|e_1) - p(f>t|e_2) \right]
	\end{align*}
	and
	\begin{align*}
		TPR_{actual} - TPR_{est} = -p(c=0)\left[p(f>t|e_1) - p(f>t|e_2) \right]
	\end{align*}
\end{proposition}
From this proposition we see that if we want accurate estimates of the FPR we need to choose a noise removal method that maximizes independence of errors with the model score, or hope that the two types of errors cancel out. In the case of a direct cleaning process, then errors will be highly correlated with the model score and we may not get reliable FPR estimates. In fact, as we will now show, at a fixed error level, direct calibrated methods provide the worst underestimates of FPR.

\begin{theorem}
\label{main-thm}
Assume that the true fraud rate $p(y=1)$ is known and that\footnote{this assumption follows naturally if label noise is class-conditional, $p(y=0, y^*=1) = 0$, but we state the theorem in terms of this weaker assumption} $p(y^*=1, c=0) = 0$. Consider some model $f$ to be validated, and let $\wt{c}$ be a calibrated direct cleaning method, i.e. one such that $p(\wt{c}=1) = p(y=1)$ and for some $\tau \in (0,1)$:
    \begin{align*}
        \wt{c}(x,y^*) = \begin{cases}
            1, &\textrm{ if } y^* = 1 \textrm{ or } f(x) > \tau \\
            0, &\textrm{ otherwise.}
        \end{cases}
    \end{align*}
Let $e = p(y=0, \wt{c}=1) = p(y=1, \wt{c} = 0)$ be the cleaning error\footnote{Recall these are equal by Lemma \ref{main-lemma}}. Then among all cleaning methods $c$ with $p(c=1) = p(y=1)$, $p(y^*=1, c=0)=0$ and $p(y=0, c=1) = e$, for any fixed model threshold $t$, the absolute error in the FPR estimate for $f$
    \begin{align*}
        \Delta FPR(c;t) := p(f > t | y=0) - p(f > t | c=0)  
    \end{align*}
takes on its maximum value at $c = \wt{c}$, and $\Delta FPR(\wt{c};t) \geq 0$.
\end{theorem}
This result shows that such a $\wt{c}$ always leads to an underestimate of the FPR, and that moreover it leads to the worst (lowest) underestimate among all cleaning methods with the same fraud rate and error (though there may still be methods that provide worse overestimates in absolute value). We prove this theorem in Appendix \ref{proofs}. The corollary for TPR is obvious, it is always maximally overestimated among comparable methods. 

\section{Experiments}

In the last section we demonstrated a need to seek alternatives to direct methods if FPR validation is a concern. One relatively simple alternative approach is to use existing micro-model approaches. The basic idea here is to let an ensemble of models vote on which examples are mislabeled. This method seems to have been first introduced for anomaly detection in network intrusion, see \cite{cretu2008casting}. In this paper, they operate in an unsupervised regime, slice their training data up according to time, and then train an anomaly detector (Anagram, Snort, or Payl) to find attacks in each slice. Their intuition is that a certain type of attack will only be present in a limited number of time slices, and any detectors \textit{not} trained on such a slice will learn to identify it as anomalous (we will not necessarily rely on this intuition). 

In the case $p(y=1)$ is known we can make this a calibrated method by choosing\footnote{possibly up to some quantization} the voting threshold such that $p(c=1) = p(y=1)$. In the case $p(y=1)$ is unknown (or calibration is otherwise unrequired/undesired), one must find another approach or rationale for choosing the threshold. If majority voting is used we may be too aggressive in our noise identification, and if consensus voting is used we may not be aggressive enough. We can likewise consider any other intermediate voting threshold, though a priori we may have no way to verify how well the method works.

Since we are mainly focused on calibrated methods, we will not further explore this here, but for the interested reader, this problem has been addressed in \cite{sabzevari2018two}, where the authors first train an ensemble of classifiers (in this case random forests) using bootstrapped samples of the data, and then let the ensemble vote on out-of-bag data in order to clean it. The algorithm is able to determine both an optimal sampling rate as well as voting threshold. In the sequel \cite{samami2020mixed}, the authors propose a more complicated hybrid approach, where the ensemble consists of different types of classifiers (one is a random forest, the next is Naive Bayes, etc.) where samples are divided into groups according to voting threshold, and the cleaning strategy (removal of samples vs. flipping of labels) is different per group. For other examples of ensemble methods in noise removal see \cite{wen2021ensemble}, \cite{tang2021detecting}. 

\subsection{Our approach}

Our approach to noise removal will be to partition the training data into disjoint slices and train a gradient boosted tree model on each slice. Thus we expect each model to be relatively weak, but relatively independent which we will find desirable.  

All of our data will have timestamps associated with it, and so there is some choice in how to select training and validation data as well as how to slice training data for micro-models. We choose not to separate training and validation data in time, and will shuffle training data before slicing for micro-models. This follows the approach of \cite{sabzevari2018two} and \cite{samami2020mixed}, vs. slices taken with respect to time as in \cite{cretu2008casting}. 

We start by training a ``base model'' on the entire set of (noisy) training data, and then for a given FPR target, validate that model on \textit{clean} validation data, taking note of the associated ``true threshold''.

We then train a model on each slice, and then have each resulting model score every example in the \textit{noisy} held-out validation set. For each example, we record whether each model scored the example as fraudulent using a simple threshold of $0.5$, and record the fraction of models voting to classify as fraud in this way. We can determine a voting threshold so that $p(c=1) = p(y=1)$, up to some quantization. It may be interesting in future work to explore the effect of out-of-time validation and/or slicing with respect to time but we do not discuss it further here.

Once the validation data has been cleaned, we take the base model trained earlier, and evaluate it on the cleaned data using the true threshold determined earlier. We measure the FPR estimate obtained on this cleaned validation data and hope that our estimate is close to the FPR target. One fact which is obvious with thought but bears mention is that \textit{underestimating} FPR in validation will lead to \textit{overshooting} FPR in production, and vice-versa. 

\subsection{Datasets}

To conduct our experiments we will work with some of our production data - which we will denote as the \texttt{cloud} dataset - as well as several publicly available fraud-related datasets. We cannot publish our production datasets for various reasons including protecting customer privacy, and will need to be intentionally vague about some characteristics here, but results are also demonstrated on the public datasets and can be reproduced\footnote{\url{https://github.com/jtittelfitz/fpr-estimation}}. In the case of our production data, examples of registration data are a representative corpus of data from customers that provided permission, and were sampled from a fixed time-period ending January 1 of a recent calendar year. Any account marked fraudulent before Jan. 1 is given an accurate label of $1$ in the data, any account marked fraudulent afterward is given the noisy label of $0$. Thus this simulates the problem of training a model on extremely recent data where label maturity may be an issue. The features are a mixture of categorical and numeric features; examples include the country associated with the IP address used during account creation, and the number of other accounts created with the same IP address. The fraud rate in this dataset is comparable to industrial trends (between $5$ and $25\%$) and the noise rate obtained by the process above is likewise in this range.

For the public datasets, we start with the datasets made available in \cite{grover2022fdb} and choose scenarios where there is some sort of timestamp associated with the event in question. We simulate noise by adding noise weighted by the event timestamp to the training and validation data: we convert these timestamps into milliseconds elapsed since the oldest example, and then use this to weight the sampling (more recent examples more likely to be noise). Overall we flip $30\%$ of the fraudulent labels to be $0$ in the training and validation data. We then train our base model, obtain a true threshold for a given FPR target, clean the validation data, and then evaluate our model (estimate FPR) on the cleaned data at the true threshold.

The benchmark datasets we use are:
\begin{itemize}
	\item \texttt{ieeecis} (\cite{ieeecis-data}) - prepared by IEEE Computational Intelligence society, this is a set of card-not-present transactions used for a Kaggle competition. The fraud rate in this dataset is $3.5\%$ and there are 67 features, 6 of which are categorical the rest numerical.
	\item \texttt{ccfraud} (\cite{ccfraud-data}) - this dataset is anonymized credit card transactions from European cardholders. The fraud rate in this dataset is very low - $0.18\%$ - and all of features are numerical (obtained by using PCA on the original set of features).
	\item \texttt{fraudecom} (\cite{fraudecom-data}) - This is a dataset of e-commerce transactions. The fraud rate in this dataset is $10.6\%$ with 6 features: 2 categorical, 3 numeric, and IP address which we do not use. 
	\item \texttt{sparknov} (\cite{sparknov-data}) - This is a simulated credit card transaction dataset generate by the Sparkov Data Generation tool. The fraud rate is $5.7\%$ and there are 17 features: 10 categorical, 6 numeric, and 1 text feature that we do not use.
\end{itemize}

We reject some of the other datasets included in \cite{grover2022fdb} as inappropriate for this study, see Appendix \ref{fdb-no-use} for explanation.

Throughout our experiments our key goal will be to clean noisy validation data in order to accurately estimate the true FPR on that same dataset. In all cases, we will use \texttt{CatBoostClassifier}s, setting number of iterations to 500 and passing in raw (i.e. un-encoded) categorical features, but otherwise doing no optimization or tuning. In all cases we will train a classifier on the entire training data and then designate this ``base model'', use it with totally clean validation data to determine the ``true threshold'' for a given FPR target, and set it aside. 

We will then compare the following approaches, cleaning the validation data in various ways:
\begin{itemize}
    \item No cleaning is done (\textbf{None}),
    \item We train an ensemble of classifiers on sliced training data (\textbf{MicroModel}),
    \item A calibrated direct method using the base model (\textbf{Direct}),
    \item CleanLab's \texttt{find\_noisy\_labels} method using the base model (\textbf{CleanLab}).
\end{itemize}
After cleaning, we will evaluate the base model with each cleaned dataset at the true threshold, taking note of the FPR estimate as well as the relative error vs. the actual FPR (i.e. $err = (FPR_{actual} - FPR_{est})/FPR_{target}$). Again we mention that FPR underestimates in validation ($\Delta FPR(c;t) \geq 0$) will lead to overshooting FPR targets in production and vice-versa.

In the cases of \textbf{MicroModel} and \textbf{Direct} we will remove the highest-scoring examples first and then proceed until we have $p(c=1)=p(y=1)$ or until the algorithm has no more suspicious examples to consider. In the case of micro-models, we use the proportion of model votes to rank examples, stopping if no model votes to remove (i.e. there is consensus that an example is not fraudulent). For \textbf{CleanLab}, we use the \texttt{return\_indices\_ranked\_by = self\_confidence} parameter in order to get a ranked list of potential noise, stopping when this list runs out of candidates (this approach often runs out of candidates before we have reached the known noise level).

In these experiments we will consider a range of FPRs that may be realistic targets in a production setting, and determine which methods yield an FPR estimate closest to our target (in practice we may want to penalize more for underestimating the FPR vs. overestimating but for now we just consider the absolute magnitude of error). The FPRs we target are $1\%, 2\%, 4\%$, and  $8\%$.

\section{Experiment Results}\label{results_section}

In nearly every experiment we see micro-models perform as well or better than CleanLab. Both perform substantially better than the \textbf{Direct} method. In the following tables the best result at each FPR is bolded. We show full results for the \texttt{cloud} dataset in this section, and include results for the other datasets in Appendix \ref{results}. Again recall that FPR underestimation in validation generally leads to FPR overshooting in production. Experiments were conducted in Python using an ml.m5.12xlarge notebook instance in AWS Sagemaker\footnote{Our experiments on the non-\texttt{cloud} datasets can be reproduced with code provided at \url{https://github.com/jtittelfitz/fpr-estimation}}. 

On the \texttt{cloud} dataset (Table \ref{tab:exact_cloud}) \textbf{MicroModel} does noticeably better than other methods at the all FPR ranges, though it does tend to underestimate. \textbf{CleanLab} performs worse than \textbf{MicroModel} but the tendency is to overestimate which may be desirable in some applications. As expected, the \textbf{Direct} method is very error-prone and leads to underestimating FPR especially at the lowest FPR targets, though at high FPR targets this aggressive approach diminishes. At high FPR \textbf{CleanLab} is probably not aggressive enough, yielding nearly the same estimates as \textbf{None}.

On the \texttt{ieeecis} dataset (Table \ref{tab:exact_ieeecis}) \textbf{MicroModel} outperforms all other methods across the range of FPRs, though the advantage diminishes at the higher FPR targets. Again, the \textbf{Direct} method exhibits drastic underestimation at low FPRs and in this case is uniformly the worst method (excluding \textbf{None}).

On the \texttt{ccfraud} dataset (Table \ref{tab:exact_ccfraud}) all methods perform comparably with a slight edge to \textbf{MicroModel} and \textbf{Direct}. This is possibly due to the fact that classes are very imbalanced in this data and so adding or cleaning noise does not do much to the calculation of the false positive rate.

On the \texttt{fraudecom} dataset (Table \ref{tab:exact_fraudecom}), \textbf{MicroModel} and \textbf{CleanLab} both perform very well across the range of FPR targets. Once again \textbf{Direct} is a substantial underestimation at low FPRs and still bad in the higher ranges.

Finally, on the \texttt{sparknov} dataset (Table \ref{tab:exact_sparknov}), \textbf{MicroModel} and \textbf{Direct} perform comparably, with \textbf{Direct} having comparable performance across the range for once. The performance of \textbf{CleanLab} is nearly as good, only slightly overestimating in each experiment.

\begin{table}[!htbp]
\caption {Results for Dataset \texttt{cloud}}
\label{tab:exact_cloud}
\begin{tabular}{lll|ll|ll|ll}
\toprule
Target FPR & \multicolumn{2}{l}{0.01} & \multicolumn{2}{l}{0.02} & \multicolumn{2}{l}{0.04} & \multicolumn{2}{l}{0.08} \\
Metric &    fpr &                     err &    fpr &                     err &    fpr &                     err &    fpr &                     err \\
\midrule
None       &  0.029 &                    1.89 &  0.045 &                    1.26 &  0.071 &                    0.79 &  0.117 &                    0.47 \\
CleanLab   &  0.023 &                    1.31 &  0.039 &                    0.97 &  0.066 &                    0.65 &  0.112 &                    0.40 \\
MicroModel &  0.002 &  \textbf{\textbf{0.83}} &  0.007 &  \textbf{\textbf{0.64}} &  0.028 &  \textbf{\textbf{0.31}} &  0.074 &  \textbf{\textbf{0.08}} \\
Direct     &  0.000 &                    1.00 &  0.000 &                    1.00 &  0.010 &                    0.76 &  0.059 &                    0.27 \\
\bottomrule
\end{tabular}

\end{table}

\section{Conclusion and Summary}

The problem of label noise has gained a lot of interest recently with many authors producing methods to clean, filter, or otherwise mitigate the effect of noise in data used to bring machine learning models into production. If we are only concerned about effects in training, direct methods can be very effective, but if we are concerned about FPR estimation in validation as well they may prove insufficient. As we proved, in a fraud detection setting where noise is class-conditional, any direct method that is calibrated to the actual fraud rate ($p(c=1) = p(y=1)$) will produce the worst underestimates of FPR compared to all other methods with comparable cleaning error. This suggests a need for future research to search for methods that not only reduce cleaning error but also produce good estimates of FPR by ensuring that cleaning error is not excessively correlated with the score of the model being validated. We hope this motivates researchers to include discussion and experiments that address this in their future work.

We also conducted a simple set of experiments using micro-models and comparing them to a calibrated direct method as well as the (non-calibrated) direct method CleanLab. We ran these experiments on our production data as well as some publicly available datasets and overall saw micro-models perform about as well as or better than CleanLab and both much better than the calibrated direct method. More sophisticated approaches can likely improve even further.

\appendix
\newpage
\section{Proofs}\label{proofs}

Here we provide proofs of the results from in Section \ref{theory}.

\begin{proof}[Proof of Lemma \ref{main-lemma}]
	Using laws of conditional probability,
	\begin{align*}
		p(e_1) &= p(y = 0, c=1)\\
			   &= p(c = 1 | y=0)p(y=0) \\
			   &= p(c=1) - p(c=1 | y=1)p(y=1).
	\end{align*}
	Now we use our assumption and the relationship between joint and conditional probability to see that
	\begin{align*}
		p(e_1)  &= p(y=1) - p(c=1, y=1) \\
			   &= p(y=1) - p(y=1 | c=1)p(c=1) \\
			   &= p(y=1 | c=0)p(c=0) \\ 
			   &= p(y=1, c=0) \\
			   &= p(e_2).
	\end{align*}
\end{proof}

\begin{proof}[Proof of Proposition \ref{main-prop}]
	Expanding using laws of conditional probability, we have
	\begin{align*}
		p(f > t | y = 0) &= p(f > t | y = 0, c = 0)p(c=0) + p(f > t | y = 0, c=1) p(c=1)\\
		p(f > t | c = 0) &= p(f > t | y = 0, c = 0)p(y=0) + p(f > t | y = 1, c=0) p(y=1).
	\end{align*}
	Under the assumption that $p(y=1) = p(c=1)$ and $p(y=0) = p(c=0)$, taking the difference of the terms above, we have
	\begin{align*}
		FPR_{actual} - FPR_{est}  &= p(c=1)\left[ p(f > t | y = 0, c=1) - p(f > t | y = 1, c=0)\right] \\
		&= p(c=1)\left[ p(f > t | e_1) - p(f > t | e_2)\right].
	\end{align*}
	The proof for TPR estimates follows in the same way.
\end{proof}

We can restate the result of this proposition in a slightly different way that highlights the relationship between the error in the estimate and the correlation between the cleaning error and the model score. In particular, the $R_i$ in the Corollary below measure the level of independence between $p(f>t)$ and $p(e_i)$ (if they are independent, $R_i$ is zero). 

\begin{corollary}
\label{main-cor}
If $c$ is a calibrated method ($p(c=1) = p(y=1)$) then
	\begin{align*}
		FPR_{actual} - FPR_{est} = \frac{p(c=1)}{p(e)}\left[R_1 - R_2 \right]
	\end{align*}
	and 
	\begin{align*}
		TPR_{actual} - TPR_{est} = -\frac{p(c=0)}{p(e)}\left[R_1 - R_2 \right]
	\end{align*}
	where $R_i := p(f>t, e_i) - p(f>t)p(e_i)$, and $p(e) := p(e_1) = p(e_2)$. 
\end{corollary}
\begin{proof}
    We start with 
    \begin{align*}
		FPR_{actual} - FPR_{est}  &= p(c=1)\left[ p(f > t | y = 0, c=1) - p(f > t | y = 1, c=0)\right]
	\end{align*}
	as obtained in the proof of Proposition \ref{main-prop}.
	If $f > t$ were independent of $e_i$ then the above would be zero. In the general case, $p(f>t | e_i) = p(f>t) + R'_i$ where 
	\begin{align*}
		R'_i = \frac{1}{p(e_i)}(p(f>t, e_i) - p(f>t)p(e_i)) = \frac{1}{p(e_i)}R_i.
	\end{align*}
	By our Lemma, we know that $p(e_1) = p(e_2)$ and so we have 
	\begin{align*}
		FPR_{actual} - FPR_{est} = \frac{p(c=1)}{p(e_1)}\left[R_1 - R_2 \right].
	\end{align*}
	The proof for TPR follows in the same way.
\end{proof}

We also have the following immediate corollaries; these are not essential to any subsequent results but we still find them interesting and worth stating.

\begin{corollary}
If $c$ is a calibrated method ($p(c=1) = p(y=1)$) then the ratio of the absolute error in TPR and FPR is given by the opposite of the odds ratio for the class probabilities:
    \begin{align*}
        \frac{TPR_{actual} - TPR_{est}}{FPR_{actual} - FPR_{est}} = -\frac{p(c=0)}{p(c=1)}.
    \end{align*}
    Further, for the ratio in the relative errors, we have
    \begin{align*}
        \frac{TPR_{actual} - TPR_{est}}{TPR_{actual}}\frac{FPR_{actual}}{FPR_{actual} - FPR_{est}} = -\frac{p(c=0)}{p(c=1)}\frac{1}{LR^+}.
    \end{align*}
    where $LR^+$ is the (true) positive predictive value of the model.
\end{corollary}
These follow from the earlier propositions and definition of $LR^+$. We now conclude this appendix with a proof of the main theorem.

\begin{proof}[Proof of Theorem \ref{main-thm}]
    First, consider the case $t \geq \tau$. In this case 
    \begin{align*}
        p(f>t | \wt{c} = 0) = p(f>t | f \leq \tau, y^*=0) = 0
    \end{align*} and $\Delta FPR(\wt{c};t)$ is clearly at its maximum value. So assume that $t < \tau$.

    As determined in the proof of Proposition \ref{main-prop}, we have
    \begin{align}
    \label{delta-def}
        \Delta FPR(c;t) = p(c=1)\left[ p(f > t | y = 0, c=1) - p(f > t | y = 1, c=0)\right].
    \end{align}
    By definition of $\wt{c}$, the condition $(y=0, \wt{c}=1)$ implies $f > \tau > t$ so that 
    \begin{align*}
        p(f > t | y = 0, \wt{c}=1) = 1
    \end{align*}
    and so clearly 
    \begin{align}
    \label{t1-est}
        p(f > t | y=0, c=1)  \leq p(f > t | y = 0, \wt{c}=1)
    \end{align}
    for any other $c$. 
    
    Now, consider the second term in \eqref{delta-def}, first noting that $p(f > t | y = 1, c=0) \leq 1$ so that $\Delta FPR(\wt{c};t) \geq 0$ is established. Next convert this conditional probability into a joint probability, i.e.
    \begin{align*}
        p(f > t | y = 1, c=0) = \frac{p(f > t , y = 1, c=0)}{p(y = 1, c=0)}.
    \end{align*}
    By assumption, we have $p(y=1, \wt{c}=0) = p(y=1, c=0)$, and so by \eqref{t1-est} it just remains to show $p(f > t , y = 1, c=0) \geq p(f > t , y = 1, \wt{c}=0)$. 
    
    Before continuing we will find it useful to note that 
    \begin{align}
    \label{note}
    p(y=1, c=0) = p(y=1, y^*=0, c=0)
    \end{align}
    for any $c$ (including $\wt{c}$ since $p(y=1, y^*=1, c=0) = 0$) by the class-conditional assumption.
    
    Now, fix some $c$ and let
    \begin{align*}
        A &= p(f>t, y=1, c=0) \\
        B &= p(f \leq t, y=1, c=0)\quad (= p(f\leq t, y=1, y^*=0, c=0) \textrm{ by } \eqref{note}) \\
        D &= p(f \leq t, y=1, y^*=0, c=1)
    \end{align*}
    and define $\wt{A}, \wt{B}, \wt{D}$ similarly in terms of $\wt{c}$. Observe that
    \begin{align}
    \label{AB}
        A + B = p(y=1, c=0) = p(y=1, \wt{c}=0) = \wt{A} + \wt{B}.
    \end{align}
    Next, observe that
    \begin{align*}
        B + D = p(f \leq t, y=1, y^* = 0) = \wt{B} + \wt{D},
    \end{align*}
    but $y^*=0, f \leq t < \tau$ implies $\wt{c} = 0$, i.e. $\wt{D}=0$ so that $B \leq \wt{B}$. This fact, along with \eqref{AB} shows that $\wt{A} \leq A$ as required.
\end{proof}
\newpage
\section{Unused Benchmark Datasets}\label{fdb-no-use}

The following datasets are included in the set of publicly available data in \cite{grover2022fdb}, but we reject them for the reasons stated below.
\begin{itemize}
	\item \texttt{twitterbot} - the data is timestamped, but the timespan of this data is over $13$ years which is not representative of the typical time intervals we are interested in.
	\item \texttt{malurl} - this dataset consists solely of URLs (some of them malicious, e.g. related to phishing), and not readily adaptable to modeling by our choice of classifier.
	\item \texttt{fakejob} - this dataset of job descriptions (some fake) does not have any timestamp information in the source.
	\item \texttt{vechicleloan} - this dataset of loan applications (some fraudulent) similarly does not have any timestamp information.
	\item \texttt{ipblock} - this dataset consists solely of IP addresses (some malicious), and not readily adaptable to modeling by our choice of classifier. 
\end{itemize}
\newpage
\section{Experiment Results}\label{results}

In this appendix we list the full results of the experiments discussed in section \ref{results_section}.
\begin{table}[!htbp]
\caption {Results for Dataset \texttt{ieeecis}}
\label{tab:exact_ieeecis}
\begin{tabular}{lll|ll|ll|ll}
\toprule
Target FPR & \multicolumn{2}{l}{0.01} & \multicolumn{2}{l}{0.02} & \multicolumn{2}{l}{0.04} & \multicolumn{2}{l}{0.08} \\
Metric &    fpr &                     err &    fpr &                     err &    fpr &                     err &    fpr &                     err \\
\midrule
None       &  0.016 &                    0.57 &  0.026 &                    0.31 &  0.047 &                    0.17 &  0.087 &                    0.09 \\
CleanLab   &  0.012 &                    0.24 &  0.023 &                    0.14 &  0.043 &                    0.09 &  0.084 &                    0.05 \\
MicroModel &  0.010 &  \textbf{\textbf{0.01}} &  0.020 &  \textbf{\textbf{0.00}} &  0.040 &  \textbf{\textbf{0.00}} &  0.081 &  \textbf{\textbf{0.01}} \\
Direct     &  0.005 &                    0.51 &  0.016 &                    0.22 &  0.036 &                    0.09 &  0.077 &                    0.04 \\
\bottomrule
\end{tabular}

\end{table}

\begin{table}[!htbp]
\caption {Results for Dataset \texttt{ccfraud}}
\label{tab:exact_ccfraud}
\begin{tabular}{lll|ll|ll|ll}
\toprule
Target FPR & \multicolumn{2}{l}{0.01} & \multicolumn{2}{l}{0.02} & \multicolumn{2}{l}{0.04} & \multicolumn{2}{l}{0.08} \\
Metric &    fpr &                     err &    fpr &                     err &    fpr &                     err &    fpr &                     err \\
\midrule
None       &  0.011 &                    0.07 &  0.021 &                    0.04 &  0.043 &                    0.07 &  0.084 &  \textbf{\textbf{0.05}} \\
CleanLab   &  0.010 &                    0.05 &  0.020 &                    0.02 &  0.042 &                    0.06 &  0.084 &  \textbf{\textbf{0.05}} \\
MicroModel &  0.010 &  \textbf{\textbf{0.02}} &  0.020 &  \textbf{\textbf{0.01}} &  0.042 &  \textbf{\textbf{0.05}} &  0.084 &  \textbf{\textbf{0.05}} \\
Direct     &  0.010 &  \textbf{\textbf{0.02}} &  0.020 &  \textbf{\textbf{0.01}} &  0.042 &  \textbf{\textbf{0.05}} &  0.084 &  \textbf{\textbf{0.05}} \\
\bottomrule
\end{tabular}

\end{table}

\begin{table}[!htbp]
\caption {Results for Dataset \texttt{ccfraud}}
\label{tab:exact_ccfraud}
\begin{tabular}{lll|ll|ll|ll}
\toprule
Target FPR & \multicolumn{2}{l}{0.01} & \multicolumn{2}{l}{0.02} & \multicolumn{2}{l}{0.04} & \multicolumn{2}{l}{0.08} \\
Metric &     fpr &            err &     fpr &            err &     fpr &            err &     fpr &            err \\
\midrule
None       &  0.0107 &           0.07 &  0.0216 &           0.08 &  0.0407 &           0.02 &  0.0830 &           0.04 \\
CleanLab   &  0.0103 &           0.03 &  0.0213 &           0.06 &  0.0404 &           0.01 &  0.0827 &  \textbf{0.03} \\
MicroModel &  0.0101 &  \textbf{0.01} &  0.0210 &  \textbf{0.05} &  0.0402 &  \textbf{0.00} &  0.0825 &  \textbf{0.03} \\
Direct     &  0.0101 &  \textbf{0.01} &  0.0210 &  \textbf{0.05} &  0.0402 &  \textbf{0.00} &  0.0825 &  \textbf{0.03} \\
\bottomrule
\end{tabular}

\end{table}

\begin{table}[!htbp]
\caption {Results for Dataset \texttt{sparknov}}
\label{tab:exact_sparknov}
\begin{tabular}{lll|ll|ll|ll}
\toprule
Target FPR & \multicolumn{2}{l}{0.01} & \multicolumn{2}{l}{0.02} & \multicolumn{2}{l}{0.04} & \multicolumn{2}{l}{0.08} \\
Metric &    fpr &                     err &    fpr &                     err &    fpr &                     err &    fpr &                     err \\
\midrule
None       &  0.012 &                    0.17 &  0.022 &                    0.09 &  0.042 &                    0.04 &  0.082 &                    0.02 \\
CleanLab   &  0.011 &                    0.06 &  0.021 &                    0.03 &  0.041 &                    0.02 &  0.081 &                    0.01 \\
MicroModel &  0.010 &  \textbf{\textbf{0.01}} &  0.020 &  \textbf{\textbf{0.00}} &  0.040 &  \textbf{\textbf{0.00}} &  0.080 &  \textbf{\textbf{0.00}} \\
Direct     &  0.010 &  \textbf{\textbf{0.01}} &  0.020 &  \textbf{\textbf{0.00}} &  0.040 &  \textbf{\textbf{0.00}} &  0.080 &  \textbf{\textbf{0.00}} \\
\bottomrule
\end{tabular}

\end{table}

\bibliography{references} 

\end{document}